\documentclass[twoside,leqno,twocolumn]{article}

\usepackage[letterpaper]{geometry}

\usepackage{color}
\usepackage{url}
\usepackage{amsmath,amsfonts,amssymb,bm,bbm}
\usepackage{subfigure}
\usepackage{adjustbox}
\usepackage{tikz}
\usepackage{cite}
\usepackage{ltexpprt}
\usepackage[inline]{enumitem}

\usepackage{stackengine} 
\usepackage{xargs} 

\usepackage[section]{algorithm}
\usepackage{algorithmicx}
\usepackage{algpseudocode}
\usepackage{caption}
\newcommand{\mat}[1]{\ensuremath{{#1}}}

\usepackage{mathtools} 
\DeclarePairedDelimiter\abs{\lvert}{\rvert}%
\DeclarePairedDelimiter\norm{\lVert}{\rVert}
\DeclarePairedDelimiter\ceil{\lceil}{\rceil}

\newcommand{\set}[1]{\ensuremath{\mathcal{#1}}}
\newcommand{\reals}[1]{\ensuremath{\mathbb{R}^{#1}}}
\newcommand{\instance}[2]{{\ensuremath{{#1}^{(#2)}}}}
\newcommand{\instsup}[3]{\instance{#1}{#2}^{#3}}

\newcommandx\LONGCOMMENT[2][2=0.5]{%
  \hfill $\triangleright~$\parbox{#2\linewidth}{\textit{#1}}%
}

\newcommand{\diagentry}[1]{\mathmakebox[1.8em]{#1}}
\newcommand{\xddots}{%
  \raise 6pt \hbox {.}
  \mkern 6mu
  \raise -1pt \hbox {.}
  \mkern 6mu
  \raise -8pt \hbox {.}
}

\usepackage{hyperref}

\usepackage{color}
\usepackage{amsmath,amsfonts,amssymb,bm,bbm}
\usepackage{subfigure}
\usepackage{adjustbox}
\usepackage{tikz}
\usepackage{cite}
\usepackage{ltexpprt}
\usepackage[inline]{enumitem}
\usepackage{hyperref}

\newtheorem{claim}{Claim}[section]

\DeclareMathOperator{\corr}{corr}

\begin{document}

\title{\Large Harmonic Alignment} \setcounter{footnote}{1}
\author{Jay S. Stanley III\thanks{Equal contribution; $^\star$Equal contribution}$\;^,$\thanks{Yale University, Appl. Math. Prog., \href{mailto:jay.stanley@yale.edu}{\texttt{jay.stanley@yale.edu}}}
\and Scott Gigante\footnotemark[2]$\;^,$\thanks{Yale University, Comp. Bio. \& Bioinf. Prog., \href{mailto:scott.gigante@yale.edu}{\texttt{scott.gigante@yale.edu}}} \and
\and Guy Wolf$\,^\star{}^,$\thanks{Universit\'{e} de Montr\'{e}al, Dept. of Math. \& Stat.; Mila -- Quebec AI Institute, \href{mailto:guy.wolf@umontreal.ca}{\texttt{guy.wolf@umontreal.ca}}}
\and Smita Krishnaswamy$\,^\star{}^,$\thanks{Yale University, Depts. of Gene. \& Comp. Sci.; Corr. author, \href{mailto:smita.krishnaswamy@yale.edu}{\texttt{smita.krishnaswamy@yale.edu}}}}

\date{}

\maketitle


\fancyfoot[R]{\scriptsize{Copyright \textcopyright\ 2020 by SIAM\\
Unauthorized reproduction of this article is prohibited}}





\begin{abstract} \small\baselineskip=9pt
We propose a novel framework for combining datasets via alignment of their intrinsic geometry. This alignment can be used to fuse data originating from disparate modalities, or to correct batch effects while preserving intrinsic data structure. Importantly, we do not assume any pointwise correspondence between datasets, but instead rely on correspondence between a (possibly unknown) subset of data features. We leverage this assumption to construct an isometric alignment between the data. This alignment is obtained by relating the expansion of data features in harmonics derived from diffusion operators defined over each dataset. These expansions encode each feature as a function of the data geometry. We use this to relate the diffusion coordinates of each dataset through our assumption of partial feature correspondence. Then, a unified diffusion geometry is constructed over the aligned data, which can also be used to correct the original data measurements. We demonstrate our method on several datasets, showing in particular its effectiveness in biological applications including fusion of single-cell RNA sequencing (scRNA-seq) and single-cell ATAC sequencing (scATAC-seq) data measured on the same population of cells, and removal of batch effect between biological samples.
\end{abstract}

\section{Introduction}
\label{sec:introduction}
High dimensional data have become increasingly common in many fields of science and technology, and with them the need for robust representations of intrinsic structure and geometry in data, typically inferred via manifold learning methods. Furthermore, modern data collection technologies often produce multisample data, which contain multiple datasets (or data batches) that aim to capture the same phenomena but originate from different equipment, different calibration, or different experimental environments. These introduce new challenges in manifold learning, as na\"ive treatment in such cases produces data geometry that largely separates different data batches into separate manifolds. Therefore, special processing is required to align and integrate the data manifolds in such cases in order to allow for the study and exploration of relations between and across multiple datasets. 

As a particular application field, we focus here on single cell data analysis, which has gained importance with the advent of new sequencing technologies, such as scRNA-seq and scATAC-seq. While numerous works have shown that manifold learning approaches are particularly effective on such data~\cite{moon2017manifold}, a common challenge in their analysis is a set of technical artifacts termed \emph{batch effects} (caused by data collection from separate experimental runs) that tend to dominate downstream analysis, unless explicitly corrected. For instance, it is often the case that na\"{i}ve data manifold construction groups such data into clusters that correspond to measurement time or equipment used, rather than by meaningful biological variations. Under such circumstances, it is necessary that batch artifacts be eliminated while actual biological differences between the samples be retained. Further, it can also be the case that different variables of data are measured on the same biological system. For example, cells from the same tissue can be measured with transcriptomic and proteomic technologies. However, the cells themselves are destroyed in each measurement, even though they are sampled from the same underlying cellular manifold. Therefore, there is no correspondence that can be established between the two sets of measurements directly.

Recent manifold learning methods focusing on multisample data often treat each dataset as a different ``view'' of the same system or latent manifold, and construct a multiview geometry (e.g., based on the popular diffusion maps framework~\cite{coifman2006diffusion}) to represent them (e.g.,~\cite{
ham2005semisupervised, 
wang2008manifold, coifman2014diffusion, 
tuia2016kernel, lederman2018learning, boumal2018heterogeneous}). Importantly, these methods often require at least partial, if not full, bijection between views (e.g., both sets of measurements conducted on the same cells), which is often impossible to obtain in experimental scenarios where data is collected asynchronously or independently. In particular, as mentioned above, genomic and proteomic data (especially at the single-cell resolution) often originate on destructive collection technologies, and thus data point (i.e., cell) correspondence becomes an impractical (if not impossible) assumption to impose on their analysis. Other works attempt to directly match data points, either in the ambient space~\cite{haghverdi2018batch, amodio2018magan} or by local data geometry~\cite{wang2009manifold}. These approaches can be very sensitive to differences in sampling density rather than data geometry, as discussed in \S\ref{sec:prob-setup}-\ref{sec:prelim}. Furthermore, a complete matching is often not feasible as certain datasets may contain distinct local phenomena (e.g., rare subpopulation only captured in one dataset but not present in the other).

In this paper, we formulate the processing of multisample data with no data point correspondence in terms of manifold alignment, and present an approach towards such alignment by bridging the geometric-harmonic framework provided by diffusion geometry~\cite{coifman2006diffusion} (\S\ref{sec:DM}) together with data feature filtering enabled by graph signal processing~\cite{shuman2013emerging} (\S\ref{sec:GFT}). Our alignment approach relies on correspondence between underlying features quantified by data collection or measurement systems, phrased here as feature correspondence, which is often more realistic that data point correspondence. Indeed, related systems often observe similar ``entities'' (e.g., cells, patients) and aim to capture related properties in them. As explained in \S\ref{sec:prob-setup} and \S\ref{sec:GFT}, we treat measured data features as manifold signals (i.e., over the data manifold) and relate them to intrinsic coordinates of a diffusion geometry~\cite{coifman2006diffusion} of each dataset, which also serve as intrinsic data harmonics. Then, as explained in \S\ref{sec:harmonic-alignment}, we leverage feature correspondence to capture pairwise relations between the intrinsic diffusion coordinates of the separate data manifolds (i.e., of each dataset). Finally, we use these relations to compute an isometric transformation that aligns the data manifolds on top of each other without distorting their internal structure.

We demonstrate the results of our method in \S\ref{sec:results} on artificial manifolds and single-cell biological data for both batch effect removal and multimodal data fusion. In each case, our method successfully aligns data manifolds such that they have appropriate neighbors both within and across the two datasets. Further, we show an application of our approach in transfer learning by applying a k-NN classifier to one unlabeled dataset based on labels provided by another dataset (with batch effects between them), and compare the classification accuracy before and after alignment. Finally, comparisons with recently developed methods such as the MNN-based method from~\cite{haghverdi2018batch} show significant improvements in performance and denoising by our harmonic alignment methods. 

\section{Problem setup}
\label{sec:prob-setup}
Let $X = \{x_1,\ldots x_M\}$ and $Y = \{y_1,\ldots y_N\}$ be two finite datasets that aim to measure the same phenomena. For simplicity, we assume that both datasets have the same number of features, i.e., $X,Y \subseteq \mathbbm{R}^n$ for some sufficiently high dimension $n$. We consider here a setting where these datasets are collected via different instruments or environment, but are expected to capture some equivalent information which can be used to align the two datasets. To leverage a manifold learning approach in such settings, we consider the common latent geometry of the data as an unknown manifold $\mathcal{M}$, which is mapped to the two feature spaces via functions $\mathbf{f},\mathbf{g} : \mathcal{M} \to \mathbbm{R}^n$ that represent the two data spaces. Namely, each data point $x \in X$ is considered as a result $x = \mathbf{f}(z) = (f_1(z),\ldots,f_n(z)) \in \mathbbm{R}^n$ for some $z \in \mathcal{M}$ and similarly each $y \in Y$ as a result of $y = \mathbf{g}(z) = (g_1(z),\ldots,g_n(z)) \in \mathbbm{R}^n$. Therefore, we aim to provide a common data representation of both datasets, which captures the geometry of $\mathcal{M}$ while allowing data fusion of $X,Y$ and integrated processing or analysis of their data features. 

We note that while we clearly do not have access to the points $z \in \mathcal{M}$ on the underlying manifold, we do have access to a finite sampling of the feature functions $f_s,g_s : \mathcal{M} \to \mathbbm{R}$, $s=1,\ldots,n$, by considering $X,Y$ as points-by-features matrices (i.e., rewriting them, by slight abuse of notation, as $X \in \mathbbm{R}^{M \times n}$ and $Y \in \mathbbm{R}^{N \times n}$, rather than finite subsets of $\mathbbm{R}^n$) and taking their corresponding columns. Further, previous multiview manifold learning methods typically consider aligned datasets, i.e., assuming that all feature functions are sampled over the same manifold points. Instead, here we remove this assumption, thus allowing independently sampled datasets, and replace it with a feature correspondence assumption. Namely, we assume the feature functions $f_s,g_s$ (for given $1 \leq s \leq n$) aim to capture similar structures in the data and should therefore share some common information, although each may also contain sensor-specific or dataset-specific bias. While this is an informal notion, it fits well with many experimental data collection settings.

\section{Preliminaries and background}
\label{sec:prelim}

\subsection{Diffusion maps}
\label{sec:DM}
To learn a manifold geometry from collected data we use the diffusion maps (DM) construction~\cite{coifman2006diffusion}, which we briefly describe here for one of the data spaces $X$, but is equivalently constructed on $Y$. This construction starts by considering local similarities, which we quantify via an anisotropic kernel 
\begin{equation}
\label{eqn:kernel}
\mathcal{K}(x_i,x_j) = \frac{\mathcal{G}(x_i,x_j)}{\|\mathcal{G}(x_i,\cdot)\|_1 \|\mathcal{G}(x_j,\cdot)\|_1},
\end{equation}
where $\mathcal{G}(x_i,x_j) = \exp \left(-\|x_i-x_j\|^2 / \sigma \right)$ is the Gaussian kernel with neighborhood radius $\sigma > 0$. As shown in \cite{coifman2006diffusion}, this kernel provides neighborhood construction that is robust to sampling density variations and enables separation of data geometry from its distribution. Next, the kernel $\mathcal{K}$ is normalized to define transition probabilities $p(x_i,y_j) = \mathcal{K}(x_i,x_j) / \|\mathcal{K}(x_i,\cdot)\|_1$ that define a Markovian diffusion process over the data. Finally, a DM is defined by organizing these probabilities in a row stochastic matrix $\mathbf{P}$ (typically referred to as the diffusion operator) as $\mathbf{P}_{ij} = p(x_i, x_j)$, and using its eigenvalues $1 = \lambda_1 \geq \lambda_2 \geq \cdots \geq \lambda_N$ and (corresponding) eigenvectors $\{\phi_j\}_{j=1}^N$ to map each $x_i \in {X}$ to diffusion coordinates $\Phi_t(x_i) = [\lambda_1^t \phi_1(x_i),\ldots,\lambda_N^t \phi_N(x_i)]^T$. The parameter $t$ in this construction represents a diffusion time or the number of transitions considered in the diffusion process. To simplify notations, we also use $\Phi_t = \{\Phi_t(x_i) : x_i \in X\}$ to denote the DM of the entire dataset $X$.  We note that in general, as $t$ increases, most of the eigenvalue weights $\lambda_j^t$, $j=1,\ldots,N$, become numerically negligible, and thus truncated DM coordinates (i.e., using only non-negligible weights) can be used for dimensionality reduction purposes, as discussed in~\cite{coifman2006diffusion}.

\subsection{Graph Fourier transform}
\label{sec:GFT}
A classic result in spectral graph theory (see, e.g.,~\cite{brooks1994geometry}) shows that the discrete Fourier basis (i.e., pure harmonics, such as sines and cosines, organized by their frequencies) can be derived as Laplacian eigenvectors of the ring graphs. This result was recently used in graph signal processing~\cite{shuman2013emerging} to define a \textit{graph Fourier transform} (GFT) by treating eigenvectors of the graph Laplacian as generalized Fourier harmonics (i.e., intrinsic sines and cosines over a graph). Further, as discussed in \cite{coifman2006diffusion,nadler2006diffusion}, diffusion coordinates are closely related to these Laplacian eigenvectors, and can essentially serve as geometric harmonics over data manifolds. 

In our case, we regard the kernel $\mathcal{K}$ from \S\ref{sec:DM} as a weighted adjacency matrix of a graph whose vertices are the data point in $X$. Then, the resulting normalized graph Laplacian is given by $\bm{\mathcal{L}} = \mathbf{I} - \mathbf{D}^{1/2} \mathbf{P} \mathbf{D}^{-1/2}$, where $\mathbf{D}$ is a diagonal matrix with $\mathbf{D}_{ii} = \|\mathcal{K}(x_i,\cdot)\|_1$. Therefore, the eigenvectors of $\bm{\mathcal{L}}$ can be written as $\psi_j = D^{1/2}\phi_j$ with corresponding eigenvalues $\omega_j = 1 - \lambda_j$. The resulting GFT of a signal (or function) $f$ over $X$ can thus be written as $\widehat{f}[j] = \langle f, \psi_j \rangle = \langle f, D^{1/2}\phi_j \rangle$. We note that here we treat either $\omega_j$  or $\lambda_j$ as providing a ``frequency'' organization of their corresponding eigenvectors $\psi_j$ or $\phi_j$ (treated as intrinsic harmonics). In the latter case, eigenvectors with higher eigenvalues correspond to lower frequencies on the data manifold, and vice versa. As noted before, the same construction of GFT here and DM in \S\ref{sec:DM} can be equivalently constructed for $Y$ as well. This frequency-based organization of diffusion coordinates derived from $X$ and $Y$, and their treatment as geometric harmonics, will be leveraged in \S\ref{sec:harmonic-alignment} to provide an isometric alignment between the intrinsic data manifolds represented by the DMs of the two datasets by also leveraging the (partial) feature correspondence assumption from \S\ref{sec:prob-setup}.

\subsection{Related work on manifold alignment}
\label{sec:related-works}

Algorithms for semi-supervised and unsupervised manifold alignment exist in classical statistics~\cite{gower1975procrustes,thompson1984canonical}, deep learning~\cite{zhu2017cyclegan,kim2017discogan,amodio2018magan} and manifold learning~\cite{haghverdi2018batch, wang2008manifold, wang2009manifold}. As mentioned in \S\ref{sec:introduction}, much work has been done on finding common manifolds between data based on known (partial) bijection between data points~\cite{coifman2014diffusion,tuia2016kernel,lederman2018learning, boumal2018heterogeneous}. In a sense, these methods can be regarded as nonlinear successors of the classic canonical correlation analysis (CCA)~\cite{thompson1984canonical}, in the same way as many manifold learning methods can be regarded as generalizing PCA. Indeed, similar to PCA, the CCA method finds a common linear projection, but on directions that maximize covariance or correlation (typically estimated empirically via known pointwise correspondence) between datasets rather than just variance within one of them. However, in this work we mainly focus on settings where no data point correspondence is available, and therefore we focus our discussion in this section on related work that operate in such settings.

One of the earliest attempts at manifold alignment (in particular, with no point correspondence), was presented in~\cite{wang2009manifold}, which proposes a linear method based on embedding a joint graph built over both datasets to preserve local structure in both manifolds. This method provides a mapping from both original features spaces to a new feature space defined by the joint graph, which is shared by both datasets with no assumption of feature correspondence. More recently, in biomedical data analysis, mutual nearest neighbors (MNN) batch correction~\cite{haghverdi2018batch} focuses on families of manifold deformations that are often encountered in biomedical data. There, locally linear manifold alignment is provided by calculating a correction vector for each point in the data, as defined by the distances from the point to all points for which it is a mutual $k$-nearest neighbor. This correction vector is then smoothed by taking a weighted average over a Gaussian kernel.

Beyond manifold learning settings, deep learning methods have been proposed to provide alignment and transfer learning between datasets. For example, cycle GANs~\cite{zhu2017cyclegan
} are a class of deep neural network in which a generative adversarial network (GAN) is used to learn a nonlinear mapping from one domain to another, and then a second GAN is used to map back to the original domain. These networks are then optimized to (approximately) satisfy cycle consistency constraints such that the result of applying the full cycle to a data point reproduces the original point. MAGAN~\cite{amodio2018magan} is a particular cycle GAN that adds a supervised partial feature correspondence to enforce alignment of two data manifolds over the mapping provided by the trained network. However, this correspondence can be disturbed by noise or sparsity in the data. 

Additionally, a similar problem exists in isometric shape matching, albeit limited to low dimensional data (i.e., shapes in at most three dimensions). For example, the method in~\cite{ovsjanikov2012maps} takes shapes with a known Laplace-Beltrami operator and aligns them using a representation of the corresponding eigenfunctions. This work was extended further in~\cite{pokrass2016sparse} to settings where ambient functions are defined intrinsically by the shape in order to learn region-region correspondences from an unknown bijection. Recent work~\cite{vestner2018pmf} has relaxed the requirement for shapes to be isometric and the need for prior knowledge of the Laplace-Beltrami operator, instead estimating the manifold with a kernel density estimate over the shape boundary. However, the application of these methods is limited to shapes, rather than a regime of point clouds as seen in high-dimensional data analysis.

In contrast, in this work we consider more general settings of aligning intrinsic data manifolds in arbitrary dimensions, while being robust to noise, data collection artifacts, and density variations. We provide a nonlinear method for aligning two datasets using their diffusion maps~\cite{coifman2006diffusion} (\S\ref{sec:DM}) under the assumption of a partial feature correspondence. Unlike MAGAN, we do not need to know in advance which features should correspond, and our results show that even with correspondence as low as 15\% we achieve good alignment between data. Further, unlike shape matching methods, we are not limited to datasets describing the boundary of a shape or dominated by density distribution. Our formulation allows us to obtain more information from datasets with partial feature correspondence than methods that assume no correspondence, but without the burden of determining in advance which or how many features correspond. To evaluate our method, in \S\ref{sec:results} we focus on comparison with MAGAN, as a leading representative of deep learning approaches, and MNN, as a leading representative of manifold learning approaches.
We note that to the best of our knowledge, the method in~\cite{wang2009manifold} is not provided with standard implementation, and our attempts at implementing the algorithm have significantly under performed other methods. For completeness, partial comparison to this method is demonstrated in Appendix~\ref{apx:wang}.

\section{Harmonic alignment}
\label{sec:harmonic-alignment}

\begin{figure*}
    \centering
    \includegraphics[width=\textwidth]{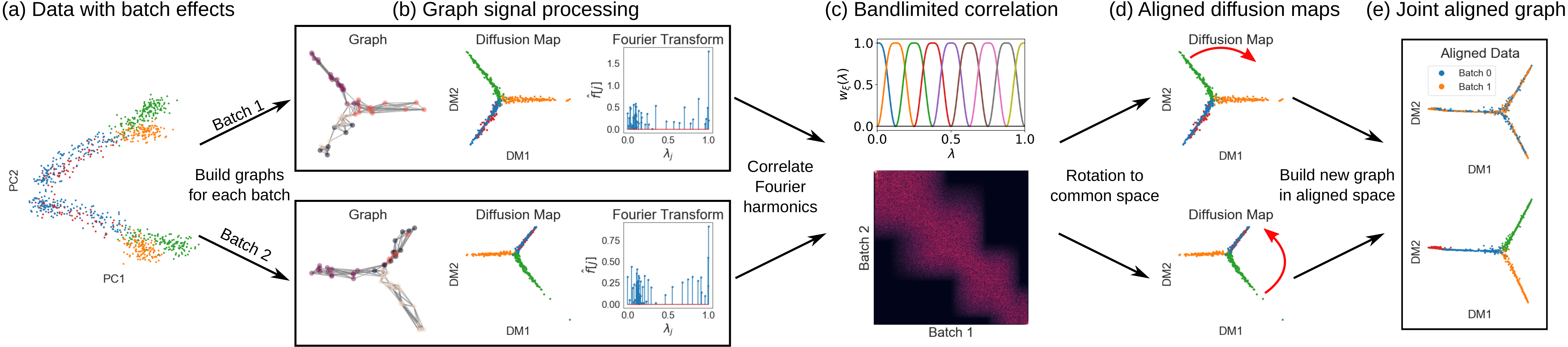}
    \caption{Schematic representation of the harmonic alignment method.}
    \label{fig:schema}
\end{figure*}

Given datasets $X,Y$, as described in \S\ref{sec:prob-setup}, we aim to construct a unified DM over both of them, which represents the global intrinsic structure of their common manifold $\mathcal{M}$ while still retaining local differences between the datasets (e.g., due to distributional differences or local patterns only available in one of the dataset). As mentioned before, global shifts and batch effects often make direct construction of such DM (or even the construction of local neighborhood kernels) over the union of both datasets unreliable and impractical. Instead, we propose here to first construct two separate DMs $\Phi^{(X)},\Phi^{(Y)}$ (based on eigenpairs $(\lambda^{(X)}_i,\phi^{(X)}_i)$, $i=1,\ldots,M$, and $(\lambda^{(Y)}_j,\phi^{(Y)}_j)$, $j=1,\ldots,N$, correspondingly), which capture the intrinsic geometry of each dataset. We then align their coordinates via an orthogonal transformation that preserves the rigid structure of the data in each DM, which is computed by orthogonalizing a correlation matrix computed between the diffusion coordinates of the two DMs. 

However, since the diffusion coordinates are associated with intrinsic notions of frequency on data manifolds (as explained in \S\ref{sec:GFT}, there is no need to compute the correlation between every pair $\phi^{(X)}_i,\phi^{(Y)}_j$, $i=1,\ldots,M$, $j=1,\ldots,N$. Indeed, leveraging the interpretation of such coordinate functions as intrinsic \emph{diffusion harmonics}, we can determine that they should not be aligned between the geometry of $X$ and $Y$ if their corresponding frequencies (i.e., captured via the eigenvalues $\lambda^{(X)}_i,\lambda^{(Y)}_j$) are sufficiently far from each other. Therefore, in \S\ref{sec:bandlimited-correlation} we describe the construction of a bandlimited correlation matrix, and then use it in \S\ref{sec:rigid-alignment} to align the two DMs. Since our alignment method is based on the treatment of DM coordinates as manifold harmonics, we call this method \emph{harmonic alignment}.

\subsection{Bandlimited correlation}
\label{sec:bandlimited-correlation}

In order to partition the diffusion harmonics into local frequency bands, we consider the following window functions, which are inspired by the itersine filter bank construction~\cite{perraudin2014designing}:
\begin{equation*}
\label{eqn:itersine}
w_{\xi}(\lambda) = \begin{cases}
\sin\left(\frac{\pi}{2}\cos^2\left(\frac{\pi}{2} (\ell\lambda-\xi) \right)\right) & \frac{\xi - 1}{\ell} \leq \lambda \leq \frac{\xi + 1}{\ell} \\
0 & \text{otherwise,}
\end{cases}
\end{equation*}
where $\xi = 0,\ldots,\ell$ and $\ell$ considered as a meta-parameter of the construction. We note that experimental evidence indicate that fine tuning $\ell$ does not significantly affect alignment quality. Each window $w_\xi(\cdot)$ is supported on an interval of length $2\ell$ around $\xi/\ell$, while decaying smoothly from $w_\xi(\xi/\ell) = 1$ to zero. Two consecutive windows (i.e., $w_\xi(\cdot),w_{\xi+1}(\cdot)$) share an overlap of half their support; otherwise (i.e., $w_\xi(\cdot),w_{\xi^\prime}(\cdot)$ with $|\xi - \xi^\prime| \geq 2$) they have disjoint supports. Finally, we recall the spectra (i.e., eigenvalues) of $P^{(X)},P^{(Y)}$ are contained in the interval $[0,1]$, which is entirely covered by $\ell+1$ window functions $w_\xi(\cdot)$, $\xi=0,\ldots,\ell$, as illustrated in Fig.~\ref{fig:schema}(c). Notice that only half the support of $w_0(\cdot)$ and $w_\ell(\cdot)$ are shown in there, since half of their support is below zero or above one, correspondingly.  

Using the soft partition defined by $w_\xi(\cdot)$, $\xi=0,\ldots,\ell$, we now define bandlimiting weights
\begin{equation}
\label{eq:weights}
w_{ij}^{(X,Y)} = \sum_{\xi=1}^{\ell}w_{\xi}(\lambda^{(X)}_i) w_{\xi}(\lambda^{(Y)}_j),
\end{equation}
for $i = 1,\ldots,M$ and $j = 1,\ldots,M$, between diffusion harmonics of $X$ and $Y$. As shown in the following lemma, whose proof appears in Appendix~\ref{apx:lemma}, these weights enable us to quantitatively identify diffusion harmonics that correspond similar frequencies and ignore relations between ones that have significantly different ones. 
\begin{lemma}
\label{lemma:bandlimiting}
The bandlimiting weights from Eqn.~\ref{eq:weights} satisfy the following properties: $w_{ij}^{(X,Y)}$ is continuous and differentiable in $\lambda_i^{(X)}$ and $\lambda_j^{(Y)}$; if $\lambda_i^{(X)} = \lambda_j^{(Y)}$ then $w_{ij} = 1$; if $|\lambda_i^{(X)} - \lambda_j^{(Y)}| > \frac{2}{\ell}$ then $w_{ij} = 0$; and the rate of change of $w_{ij}^{(X,Y)}$ w.r.t.\ $|\lambda_i^{(X)} - \lambda_j^{(Y)}|$ is bounded by $O(\ell)$.
\end{lemma}
Next, we use the weights from Eqn.~\ref{eq:weights} to construct a $M \times N$ bandlimited correlation matrix $C$ defined as 
\begin{equation}
\label{eq:band-corr}
\left[C\right]_{ij} = w_{ij}^{(X,Y)} \corr\left(\phi^{(X)}_i,\phi^{(Y)}_j\right)
\end{equation}
for $i=1,\ldots,M$ and $j=1,\ldots,N$, which only considers correlations between diffusion harmonics within similar frequency bands.

Finally, for each $i,j$ with nonzero weight $w_{ij}^{(X,Y)}$, we now need to compute a correlation between the diffusion harmonics $\phi^{(X)}_i,\phi^{(Y)}_j$. If we had partial data point correspondence, as is assumed in many previous work (e.g.,~\cite{coifman2014diffusion,lederman2018learning}), we could estimate such correlation directly from matching parts of the two datasets. However, in our case we do not assume any a priori matching between data points. Instead, we rely on the assumed feature correspondence and leverage the GFT from \S\ref{sec:GFT} to express the harmonics $\phi^{(X)}_i,\phi^{(Y)}_j$ in terms of the data features via their Fourier coefficients. Namely, we take the GFT of the data features $f_s,g_s$, $s=1,\ldots,n$ (i.e., the ``columns'' of the points-by-features representation of $X,Y$ as $M \times n, N \times n$ data matrices, correspondingly), and use them to represent $\phi^{(X)}_i,\phi^{(Y)}_j$ by the $n$ dimensional vectors $\hat{x}_i = (\hat{f}_1[i],\ldots,\hat{f}_n[i])^T$ and $\mathbf{\hat{y}_j} = (\hat{g}_1[j],\ldots,\hat{g}_n[j])^T$, correspondingly. Then, we compute a correlation between the harmonics $\phi^{(X)}_i,\phi^{(Y)}_j$ indirectly via a correlation between $\mathbf{\hat{x}_i},\mathbf{\hat{y}_j}$. For simplicity, and by slight abuse of terminology, we use an inner product in lieu of the latter, to define
\begin{equation*}
\corr\left(\phi^{(X)}_i,\phi^{(Y)}_j\right) = \langle \mathbf{\hat{x}_i},\mathbf{\hat{y}_j} \rangle \, .
\end{equation*}
Therefore, together with Eqn.~\ref{eq:band-corr}, our bandlimited correlation matrix is given by $\left[C\right]_{ij} = w_{ij}^{(X,Y)} \langle \mathbf{\hat{x}_i},\mathbf{\hat{y}_j} \rangle = \sum_{\xi=1}^{\ell} \langle w_{\xi}(\lambda^{(X)}_i) \mathbf{\hat{x}_i}, w_{\xi}(\lambda^{(Y)}_j) \mathbf{\hat{y}_j} \rangle$.

\subsection{Rigid alignment}
\label{sec:rigid-alignment}
Given the bandlimited correlation matrix $C$, we use its SVD given by $C = U \Sigma V^T$ to obtain its nearest orthogonal approximation $\mathbf{T} = U V^T$ (e.g., as shown in~\cite{schonemann1966generalized}) that defines an isometric transformation between the diffusion maps of the two samples, which we refer to as harmonic alignment. Finally, we can now compute a unified diffusion map, which can be written in (block) matrix form as
\begin{equation}
\label{eqn:aligned-DM}
\Phi_t^{(X,Y)} = 
    \begin{bmatrix}
    \Phi_{0}^{(X)} & \Phi_0^{(X)}~\mathbf{T}  \\
    \Phi_{0}^{(Y)}~\mathbf{T}^{T} & \Phi_0^{(Y)}
    \end{bmatrix} \;
    \begin{bmatrix}
        \Lambda^{(X)} & 0 \\
        0 & \Lambda^{(X)}
    \end{bmatrix}^{\textstyle{t}} \, ,
\end{equation}
where $\Lambda^{(X)},\Lambda^{(Y)}$ are diagonal matrices with the diffusion eigenvalues $\{\lambda_i^{(X)}\}_{i=1}^N$,$\{\lambda_j^{(Y)}\}_{j=1}^M$ (correspondingly) as their main diagonal, and $t$ is an integer diffusion time parameter as in \S\ref{sec:DM}.
A summary of the described steps is presented in Appendix~\ref{apx:algorithm}.
While this construction is presented here in terms of two datasets for simplicity, it can naturally be generalized to multiple datasets by considering multiple blocks (rather than the two-by-two block structure in Eqn.~\ref{eqn:aligned-DM}), based on orthogonalizing pairwise bandlimited correlations between datasets. This generalization is discussed in detail in Appendix~\ref{apx:algorithm/multi}.

Finally, given aligned DMs in $\Phi_t^{(X,Y)}$, we can construct a new neighborhood kernel over their coordinates (i.e., in terms of a combined diffusion distance) and build a robust unified diffusion geometry over the entire entire data in $X \cup Y$ that is invariant to batch effects and also enables denoising of data collection artifacts that depend on environment or technology rather than the underlying measured phenomena. This diffusion geometry can naturally be incorporated in diffusion-based methods for several data processing tasks, such as dimensionality reduction \& visualization~\cite{moon2017visualizing}, denoising \& imputation~\cite{van2018recovering}, latent variable inference~\cite{lederman2018learning
}, and data generation~\cite{lindenbaum2018sugar}. In particular, in \S\ref{subsec:biology} we demonstrate the application of harmonic alignment to batch effect removal and multimodal data fusion with various single-cell genomic technologies.

\color{black}

\section{Numerical results}
\label{sec:results}

\subsection{Artificial feature corruption}
\label{subsec:featurecorruption}

\begin{figure*}[!htb]
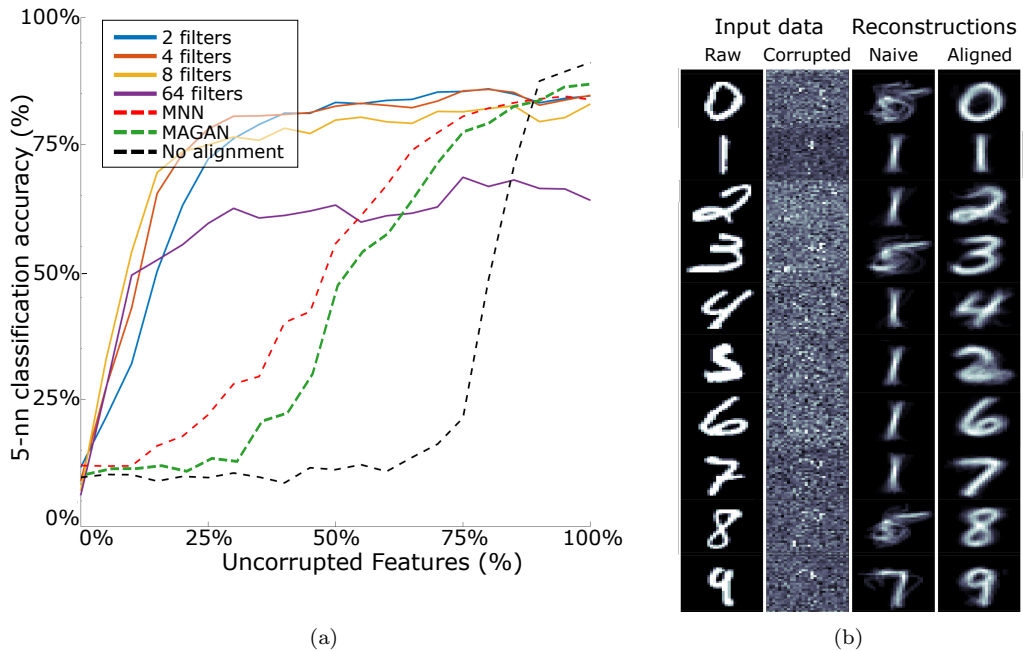
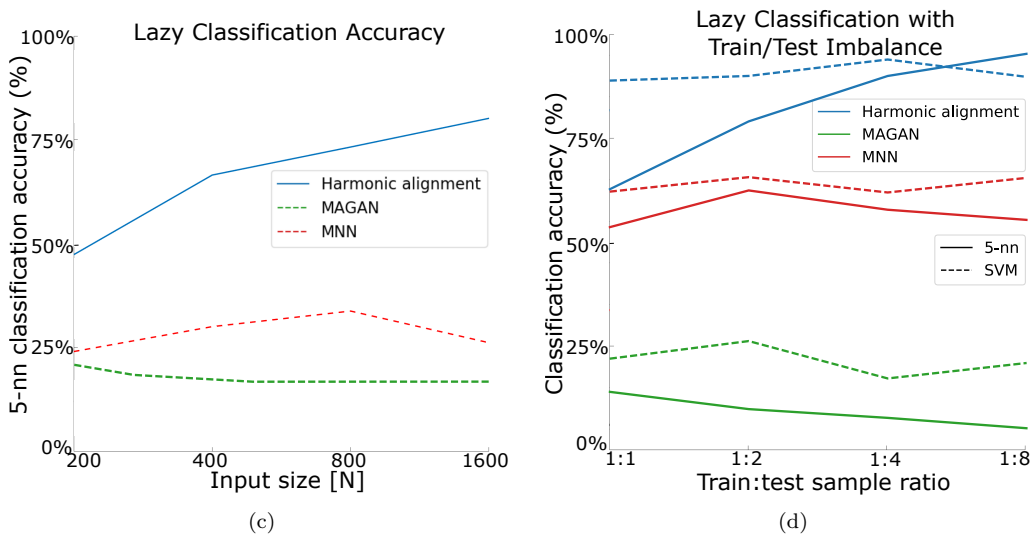

    \centering
    \parbox{0.8\linewidth}{
    \subfigure[]{\adjustbox{
    width=0.64\linewidth
    }{\input{Figures/MNIST/mnist_columncorrespondence.tex}}\label{subfig:KNNrecovery}} \hfill
    \subfigure[]{\adjustbox{
    width=0.335\linewidth
    }{\input{Figures/MNIST/mnist_recons.tex}}\label{subfig:KNNrecons}} \hfill
    }
    \parbox{0.8\linewidth}{
    \subfigure[]{\adjustbox{
    width=0.485\linewidth
    }{\input{Figures/Comparison/accuracy.tex}} \label{subfig:compclassification}} \hfill
    \subfigure[]{\adjustbox{
    width=0.485\linewidth
    }{\input{Figures/Comparison/mnist_databalance.tex}}\label{subfig:transferlearning}}}
    \caption{Recovery of k-nearest neighborhoods under feature corruption. Mean over 3 iterations is reported for each method.
    \protect\subref{subfig:KNNrecovery} At each iteration, two sets $X$ and $Y$ of $1000$ points were sampled from MNIST.  $Y$ was then distorted by a $784 \times 784$ corruption matrix $\mathbf{O}_p$ for various identity percentages $p$ (\S\ref{subsec:featurecorruption}).  Subsequently, a lazy classification scheme was used to classify points in  $Y\mathbf{O}_p$ using a 5-nearest neighbor vote from $X$. Results for harmonic alignment with $\ell \in \{2, 4, 8, 64\}$ (\S\ref{sec:bandlimited-correlation}), mutual nearest neighbors (MNN), and classification without alignment are shown.
    \protect\subref{subfig:KNNrecons} Reconstruction of digits with only 25\% uncorrupted features.  Left: Input digits. Left middle: 75\% of the pixels in the input are corrupted. Right middle: Reconstruction without harmonic alignment. Right: Reconstruction after harmonic alignment.
    \protect\subref{subfig:compclassification} Lazy classification accuracy relative to input size with unlabeled randomly corrupted digits with 35\% preserved pixels.
    \protect\subref{subfig:transferlearning} Transfer learning performance.  For each ratio, 1K uncorrupted, labeled digits were sampled from MNIST, and then 1K, 2K, 4K, and 8K (x-axis) unlabeled points were sampled and corrupted with 35\% column identity.}
    \label{fig:corruptions}
\end{figure*}

To demonstrate the accuracy of harmonic alignment, we assess its ability to recover $k$-nearest neighborhoods after random feature corruption, and compare it to MNN~\cite{haghverdi2018batch} and MAGAN~\cite{amodio2018magan}, which are leading manifold- and deep-learning methods respectively, as discussed in \S\ref{sec:related-works}. To this end, we drew two random samples $X$ and $Y$ of $1000$ MNIST digit images, each of which is a $784$-dimensional vector. For each trial, we generate a random orthogonal $784 \times 784$ corruption matrix $\mathbf{O}_0$. To vary the amount of feature corruption, we produce partial corruption matrices $\mathbf{O}_p$ by randomly substituting $p\%$ of the columns in $\mathbf{O}_0$ with columns of the identity matrix. Right multiplication of $Y$ by these matrices yields corrupted images with only $p\%$ preserved pixels (Fig.~\ref{subfig:KNNrecons}, `Corrupted').

To assess the alignment of the corrupted images $Y\mathbf{O}_p$ to the uncorrupted images $X$, we perform lazy classification on digits (i.e., rows) in $Y\mathbf{O}_p$ by using the labels of each aligned image's $k$ nearest neighbors in $X$. The results of this experiment, performed for $p = \{ 0,5,10,\ldots 95,100\}$, are reported in Fig.~\ref{subfig:KNNrecovery}. For robustness, at each $p$ we sampled three different non-overlapping pairs $X, Y$, and for each pair we sampled three random $\mathbf{O}_p$ matrices. It should be noted that while we report results in terms of mean classification accuracy, we do not aim to provide an optimal classifier here. Our evaluation merely aims to provide a quantitative assessment of neighborhood quality before and after alignment. We regard a lazy learner as ideal for such evaluation since it directly exposes the quality of data neighborhoods, rather than obfuscate it via a trained model. For comparison, results for harmonic alignment with a SVM classifier are shown in \S\ref{subsec:transfer} and Fig.~\ref{subfig:transferlearning}.

In general, none of the methods recovers $k$-nearest neighborhoods under total corruption, showing 10\% accuracy for very small $p$, essentially giving random chance accuracy. Note that this case clearly violates our (partial) feature correspondence assumption. However, when using sufficiently many bandlimited filters, harmonic alignment quickly recovers over $80\%$ accuracy and consistently outperforms both MNN and MAGAN, except under under very high correspondence (i.e., when $\mathbf{O}_p \approx \mathbf{I}$). The method proposed by~\cite{wang2009manifold} was excluded since it did not show improvement over unaligned classification, but is discussed in supplemental materials for completeness. We note that the performance of harmonic alignment is relatively invariant to the choice of $\ell$, with the exclusion of extremely high values. For the remainder of the experiments, we fix $\ell = 8$. All experiments use the default parameter $t = 1$.

Next, we examined the ability of harmonic alignment to reconstruct the corrupted data (Fig.~\ref{subfig:KNNrecons}). We performed the same corruption procedure with $p=25\%$ and selected one example of each MNIST digit. Ground truth from $Y$ and corrupted result $Y\mathbf{O}_{25}$ are shown in Fig.~\ref{subfig:KNNrecons}. Then, reconstruction was performed by setting each pixel in a new image to the dominant class average of the $10$ nearest neighbors from $X$. In the unaligned case, we see that most examples give smeared fives or ones; this is likely a random intersection formed by $X$ and $Y\mathbf{O}_{25}$. On the other hand, reconstructions produced by harmonic alignment resemble the original input examples.

Finally, in Fig.~\ref{subfig:compclassification}, we consider the effect of data size on obtained alignment. To this end, we fix $p=35\%$ and vary the size of the two aligned datasets. We compare harmonic alignment, MNN, and MAGAN on input sizes ranging from 200 to 1600 MNIST digits, while again using lazy classification accuracy to measure neighborhood preservation and quantify alignment quality. The results in Fig.~\ref{subfig:transferlearning} show that both MNN and MAGAN are not significantly affected by dataset size, and in particular do not improve with additional data. Harmonic alignment, on the other hand, not only outperforms them significantly -- its alignment quality increases monotonically with input size.

\begin{figure*}[!htbp]
    \centering
    \subfigure[]{\adjustbox{width=0.17\linewidth}{\input{Figures/Biology/noisy.tex}}\label{subfig:noisy}} \hfill
    \subfigure[]{\adjustbox{width=0.17\linewidth}{\input{Figures/Biology/unaligned.tex}}\label{subfig:biounaligned}} \hfill
    \subfigure[]{\adjustbox{width=0.17\linewidth}{\input{Figures/Biology/aligned.tex}} \label{subfig:bioaligned}} \hfill
    \subfigure[]{\adjustbox{width=0.42\linewidth}{\includegraphics{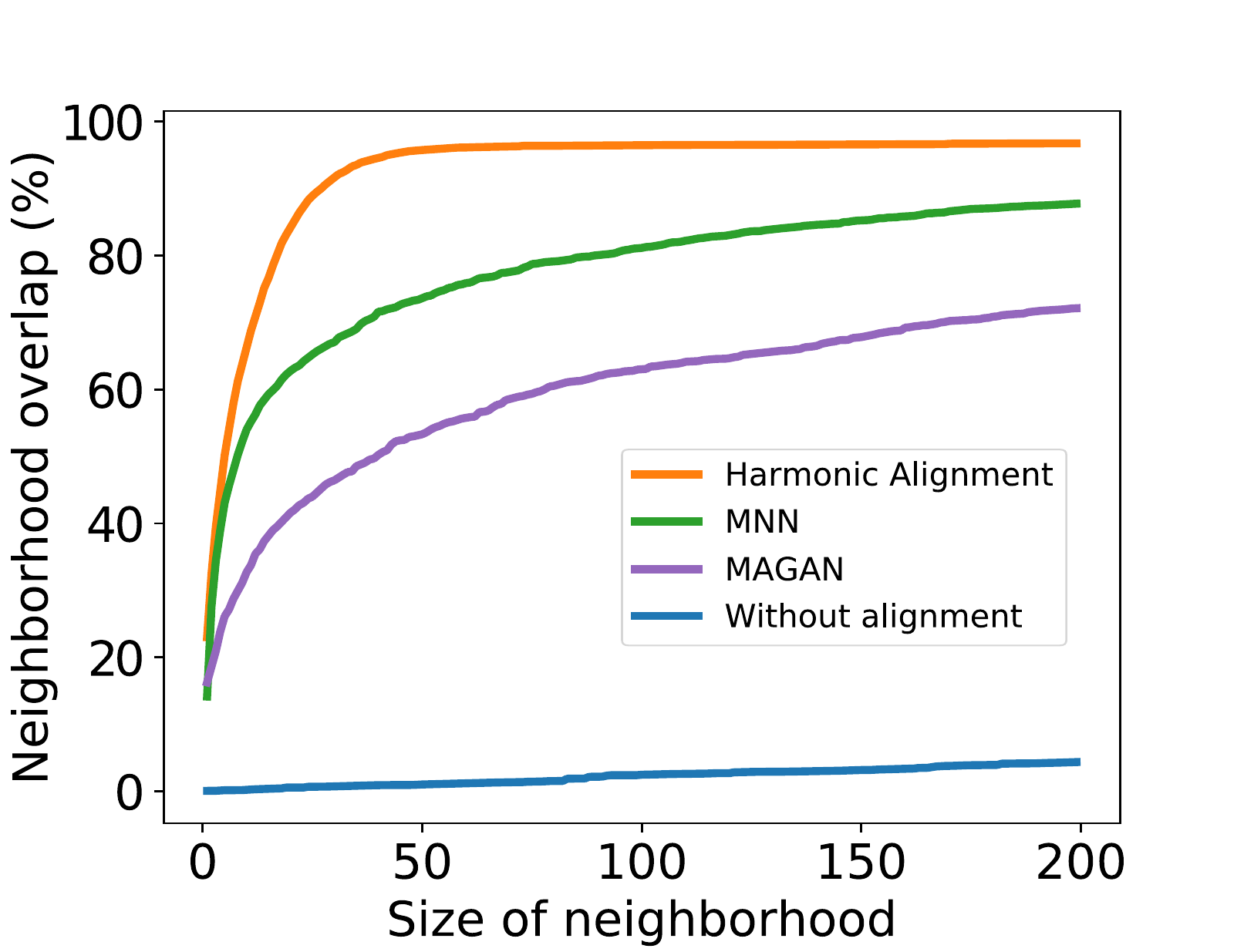}} \label{subfig:rnaseq_atacseq}}
    \caption{\protect\subref{subfig:noisy}-\protect\subref{subfig:bioaligned} \textit{Batch effect removal.} 4K cells were subsampled from two single-cell mass cytometry immune profiles on blood samples of two patients infected with Dengue fever. \emph{Top:} Both patients exhibit heightened IFN$\gamma$ (x-axis), a pro-inflammatory cytokine associated with TNF$\alpha$ (y-axis)
    \emph{Bottom:} IFN$\gamma$ histograms for each batch.
    \protect\subref{subfig:noisy} Data before denoising.
    \protect\subref{subfig:biounaligned} Denoising of unaligned data enhances a technical effect between samples in IFN$\gamma$. \protect\subref{subfig:bioaligned} Harmonic alignment corrects the IFN$\gamma$ shift.
    \protect\subref{subfig:rnaseq_atacseq} \textit{Multimodal data fusion.} Overlap of cell neighborhoods from joint gene expression and chromatin profiling of single cells. Harmonic alignment most accurately recovers the pointwise relationship between the manifolds.}
    \label{fig:biology}
\end{figure*}

\subsection{Transfer learning}
\label{subsec:transfer}

An interesting use of manifold alignment algorithms is transfer learning.  In this setting, an algorithm is trained to perform well on a small (e.g., pilot) dataset, and the goal is to extend the algorithm to a new larger dataset (e.g., as more data is being collected) after alignment. In this experiment, we first randomly selected $1,000$ uncorrupted examples of MNIST digits, and constructed their DM to use as our training set. Next, we took $65\%$-corrupted unlabeled points (\S\ref{subsec:featurecorruption}) in batches of $1,000$, $2,000$, $4,000$, and $8,000$, as a test set for classification using the labels from the uncorrupted examples. As shown in~\ref{subfig:transferlearning}, with a 5-nearest neighbor lazy classifier, harmonic alignment consistently improves as the dataset gets larger, even with up to \textit{eight} test samples for every one training sample. When the same experiment is performed with a linear SVM, harmonic alignment consistently outperforms other methods with performance being independent of test set size (or train-to-test ratio). This is due to the increased robustness and generalization capabilities of trained SVM. Further discussion of transfer learning is given in the supplementary materials. In addition to showing the use of manifold alignment in transfer learning, this example also demonstrates the robustness of our algorithm to imbalance between samples. 

\subsection{Biological data}
\label{subsec:biology}

\subsubsection{Batch effect correction}
To illustrate the need for robust manifold alignment in computational biology, we turn to a simple real-world example from \cite{amodio2018exploring} (Fig.~\ref{fig:biology}).  This dataset was collected by mass
cytometry (CyTOF) of peripheral blood mononuclear cells from patients who contracted dengue fever~\cite{amodio2018exploring}.

The canonical response to dengue infection is upregulation of interferon gamma (IFN$\gamma$).\cite{
chakravarti2006circulating
} During early immune response, IFN$\gamma$ works in tandem with acute phase cytokines such as tumor necrosis factor alpha (TNF$\alpha$) to induce febrile response and inhibit viral replication~\cite{ohmori1997synergy}. We thus expect to see upregulation of these two cytokines together. 

In Fig.~\ref{subfig:noisy}, we show the relationship between IFN$\gamma$ and TNF$\alpha$ without denoising. Note that there is a substantial difference between the IFN$\gamma$ distributions of the two samples (Earth Mover's Distance [EMD] = 2.699). In order to identify meaningful relationships in CyTOF data, it is common to denoise it first~\cite{moon2017manifold}. We used a graph low-pass filter proposed in~\cite{van2018recovering} to denoise the cytokine data. The results of this denoising are shown in Fig.~\ref{subfig:biounaligned}. This procedure introduced more technical artifacts by enhancing differences between batches, as seen by the increased EMD (3.127) between IFN$\gamma$ distributions of both patients. This is likely due to substantial connectivity differences between the two batch submanifold in combined data manifold. 

Next, we performed harmonic alignment of the two patient profiles (Fig.~\ref{subfig:bioaligned}). Harmonic alignment corrected the difference between IFN$\gamma$ distributions and restored the canonical correlation of IFN$\gamma$ and TNF$\alpha$ (EMD=0.135). This example illustrates the utility of harmonic alignment for biological data, where it can be used for integrated analysis of data collected across different experiments, patients, and time points.

\subsubsection{Multimodal Data Fusion}
\label{sec:multimodal-fusion}

Since cells contain numerous types of components that are informative of their state (genes, proteins, epigenetics), modern experimental technologies are starting to measure of each of these components separately at the single cell level. Since most single-cell assays are destructive, it is challenging or impossible to obtain all desired measurements in the same cells. It is therefore desirable to perform each assay on a subset of cells from a single sample, and align these datasets \textit{in silico} to obtain a pseudo-joint profile of the multiple data types.

To demonstrate the utility of harmonic alignment in this setting, we use a dataset obtained from~\cite{cao2018joint} of 11,296 cells from adult mouse kidney collected by a joint measurement technique named sci-CAR, which measures \textit{both} gene expression (scRNA-seq) and chromatin accessibility (scATAC-seq) in the same cells simultaneously. The datasets are normalized separately as in~\cite{van2018recovering}, using a square root transformation for the scRNA-seq and a log transformation with a pseudocount of 1 for the scATAC-seq data, and finally the dimensionality of each dataset is reduced to 100 using truncated SVD. After randomly permuting the datasets to scramble the correspondence between them, we align the two manifolds in order to recover the known bijection between data modalities. Let $f(i) \in F$ be the scRNA-seq measurement of cell $i$, and $g(i) \in G$ be the scATAC-seq measurement of cell $i$. Fig.~\ref{subfig:rnaseq_atacseq} shows the average percentage overlap of neighborhoods of $f(i)$ in $F$ with neighborhoods of $g(i)$ in $G$, before and after alignment with: MAGAN, MNN and Harmonic Alignment. Harmonic Alignment most accurately recovers cell neighborhoods, thereby allowing the generation of \textit{in silico} joint profiles across data types and obviating the need for expensive or infeasible \textit{in vitro} joint profiling.

\section{Conclusion}

We presented a novel method for processing multisample data, which contains multiple sampled datasets that differ by global shifts or \emph{batch effects}. To perform data fusion and provide a single stable representation of the entire data, we proposed to learn an intrinsic diffusion geometry of each individual datasets and then align them together using the duality between diffusion coordinates used in manifold learning and manifold harmonics used in graph signal processing. While previous methods for data manifold alignment relied on known bijective correspondence between data points, our method replaces such strict requirement by considering feature correspondence in the sense that corresponding features across samples or datasets should have similar intrinsic regularity (or ``frequency'' composition) on the diffusion geometry of each sampled dataset in the data. Our \textit{harmonic alignment} leverages this understanding to compute cross-dataset similarity between manifold harmonics, which is then used to construct an isometric transformation that aligns the data manifolds. Results show that our method is effective in resolving both artificial misalignment and biological batch effects, thus allowing data fusion and transfer learning. We expect future applications of harmonic alignment to include, for example, the use of multimodal data fusion to understand complex molecular processes through three or more different data modalities.

\section*{Acknowledgments}
\noindent This work was partially funded by: the Gruber Foundation [\emph{S.G.}]; IVADO (l'institut de valorisation des donn\'{e}es) [\emph{G.W.}]; Chan-Zuckerberg Initiative grants 182702 \& CZF2019-002440 [\emph{S.K.}]; and NIH grants R01GM135929 \& R01GM130847 [\emph{G.W.,S.K.}].

\bibliographystyle{ieeetr}
\bibliography{main}

\appendix
\onecolumn
\section{Algorithm}
\label{apx:algorithm}

\subsection{Standard implementation}
\label{apx:algorithm/naive}

Consider the standard setting of harmonic alignment as described in Section~\ref{sec:harmonic-alignment}. We provide here a detailed description of the Harmonic Alignment algorithm in pseudocode.

Let $X$ and $Y$ be collections of data points $\vec{x} \in \reals{d}$. Let the diffusion time $t$ and band count $\ell$ be positive integers. Let $\instance{\sigma}{X}, \instance{\sigma}{Y}$ be bandwidth functions (generally either positive constants or a function of the distance from a point to its \textit{k}th nearest neighbor) and the anisotropy $\instance{q}{X}, \instance{q}{Y} \in [0,1]$. Let the kernel parameters $\instance{\set{K}}{i} = \left\{ \instance{\sigma}{i}, \instance{q}{i} \right\}$.

Then the aligned data is given by $\instance{\Phi_t}{X,Y} = \mathrm{ HarmonicAlignment}\left(X, Y,\left\{\instance{\set{K}}{X}, \instance{\set{K}}{Y}\right\},t,\ell \right)$, where the first $\#\abs X$ points in $\instance{\Phi_t}{X,Y}$ are the aligned coordinates of $X$ and the last $\#\abs Y$ points in $\instance{\Phi_t}{X,Y}$ are the aligned coordinates of $Y$.

\begin{algorithm}[htbp]
\begin{tabular}{l l l}

\textbf{Input:} & Data sets & $\left\{X, Y\right\}$ \\ & Kernel parameters & $\set{K} = \left\{\instance{\set{K}}{X}, \instance{\set{K}}{Y}\right\}$ \\ & Alignment diffusion time & $t$ \\ & Alignment band count & $\ell$ \\
\textbf{Output:} & Aligned diffusion map & $\instance{\Phi_t}{X,Y}$ 
\end{tabular} 
    \begin{algorithmic}[1]
    \For{Z=X,Y}
    \State $\instance{\set{G}}{Z} \leftarrow \left\{\instance{W}{Z}, \instance{\set{L}}{Z}, \instance{D}{Z}\right\} \leftarrow \mathrm{GaussKernelGraph}\left(Z, \instance{\set{K}}{Z}\right)$\LONGCOMMENT{Alg.~\ref{alg:gausskernel}}[0.3]
    \State $\instance{\Psi}{Z}, \instance{\Lambda}{Z} \leftarrow \mathrm{SVD}\left(I-\instance{\set{L}}{Z}\right)$ \LONGCOMMENT{Graph Fourier basis} [0.3]
    \EndFor
\State $\instance{\Phi_t}{X,Y} \leftarrow \mathrm{Align}\left(\left\{\instance{\Psi}{X}, \instance{\Lambda}{X}, \instance{D}{X}\right\}, \left\{\instance{\Psi}{Y}, \instance{\Lambda}{Y}, \instance{D}{Y}\right\},t, \ell\right)$ \LONGCOMMENT{Alg.~\ref{alg:align}}[0.3]
    \State \Return $\instance{\Phi_t}{X,Y}$

    \end{algorithmic}

    \captionof{algorithm}[function HarmonicAlignment: Align the diffusion maps of two datasets.]{function HarmonicAlignment$\left(X, Y,\set{K},t,\ell \right)=\protect\instance{\Phi_t}{X,Y}$ }
    \label{alg:harmonicalignment}
\end{algorithm}

\begin{algorithm}[htbp]
\begin{tabular}{l l l}

\textbf{Input:} & Dataset & $X = \{\vec{x}_1,\ldots, \vec{x}_N : x \in \chi\}\subseteq \reals{d}$ \\ & Bandwidth function & $\sigma : X \mapsto \reals{}$ \\
\textbf{Output:} & Kernel matrix &\mat{W}\\
& Degree matrix & $\mat{D}$ \\ & Normalized Laplacian & $\set{L} $ 
\end{tabular} 
    \begin{algorithmic}[1]
    \For{$i=1,N$}
    \State $D(i,j) \leftarrow 0$
    \For{$j=1,N$}
        \State $W(i,j)\leftarrow \frac{1}{2} \left( \exp{\left(\frac{-\norm{\vec{x}_i-\vec{x}_j}_2^2}{2\epsilon(x_i)}\right)} + \exp{\left(\frac{-\norm{\vec{x}_i-\vec{x}_j}_2^2}{2\epsilon(x_j)}\right)} \right)$ \LONGCOMMENT{Symmetric kernel}[0.4]
        \State $D(i,i) \leftarrow D(i,i)+W(i,j)$ \LONGCOMMENT{Degrees}[0.4]
    \EndFor
    \EndFor
    \State $\set{L} \leftarrow I - D^{-1/2} W D^{-1/2}$ \LONGCOMMENT{Normalized graph {Laplacian}}[0.4]
    \State \Return $\{\mat{W},\mat{D},\set{L}\}$
    \end{algorithmic}

    \captionof{algorithm}[function GaussKernelGraph: Apply a Gaussian kernel to learn a graph on data.]{function GaussKernelGraph$(X,\epsilon) = \{\mat{W},\mat{D},\set{L}\}$ \\ Apply a Gaussian kernel to a dataset and compute the corresponding graph matrices.}    
    \label{alg:gausskernel}

\end{algorithm}

\begin{algorithm}[htbp]
\begin{tabular}{l l l}

\textbf{Input:} & Laplacian eigensystems and degrees & $\instance{\set{G}}{X}, \instance{\set{G}}{Y} : \instance{\set{G}}{Z} = \left\{\instance{\Psi}{Z}, \instance{\Lambda}{Z}, \instance{D}{Z}\right\}$ \\ & Alignment diffusion & $t$ \\ & Alignment band count & $\ell$ \\
\textbf{Output:} & Aligned diffusion map & $\instance{\Phi_t}{X,Y}$ 
\end{tabular} 
    \begin{algorithmic}[1]
    \For{Z=X,Y}
    \State $\instance{\Psi}{Z} \leftarrow \instance{\Psi}{Z} \setminus \instance{\psi_1}{Z};~\instance{\Lambda}{Z}\leftarrow \instance{\Lambda}{Z}\setminus\instance{\lambda_1}{Z}$
    \State $\instance{\Phi^0}{Z}  \leftarrow  {\instance{D}{Z}}^{1/2}\instance{\Psi}{Z}$
    \State $\hat{Z} \leftarrow  \instance{\Psi}{Z}^T Z$ \LONGCOMMENT{\mbox{Graph Fourier} \mbox{transform}\hfill}[0.22]
    \EndFor
    \State $\instance{w}{X,Y}\leftarrow \mathrm{BandlimitingWeights}\left(\protect\instance{\Lambda}{X},\protect\instance{\Lambda}{Y}, \ell\right)$\LONGCOMMENT{Alg.~\ref{alg:bandlimitingweights}}[0.22]
    \For{$i=2,N_1$}
    \For{$j=2,N_2$}
    \State $C\left(i-1,j-1\right) \leftarrow \instance{w_{ij}}{X,Y} \left\langle \hat{X}(i-1,:), \hat{Y}(j-1,:) \right\rangle$\LONGCOMMENT{Bandlimited \mbox{correlations}}[0.22]
    \EndFor
    \EndFor
    \State $U,S,V \leftarrow \mathrm{SVD}(C)$ 
    \State $T \leftarrow UV^T$ \LONGCOMMENT{Orthogonalization Sec.~\ref{sec:rigid-alignment}\hfill}[0.22]
    \State $\instance{\Phi_t}{X,Y}  \leftarrow \begin{bmatrix}
    \instance{\Phi^0}{X} & \instance{\Phi^0}{X}~\mathbf{T}  \\
    \instance{\Phi^0}{Y}~\mathbf{T}^{T} & \instance{\Phi^0}{Y}
    \end{bmatrix} \;
    \begin{bmatrix}
        \instance{\Lambda}{X} & 0 \\
        0 & \instance{\Lambda}{Y}
    \end{bmatrix}^{\textstyle{t}}$\
    \State \Return $\instance{\Phi_t}{X,Y}$

    \end{algorithmic}

    \captionof{algorithm}[function Align: Compute and apply an alignment matrix for two diffusion maps.]{function Align$\left(\protect\instance{\set{G}}{X}, \protect\instance{\set{G}}{Y}, t, \ell\right)=\protect\instance{\Phi_t}{X,Y}$ \\ Compute and apply an alignment matrix to two diffusion maps.}
    \label{alg:align}
\end{algorithm}

\begin{algorithm}[htbp]
\begin{tabular}{l l l}
\textbf{Input:} & Normalized Laplacian eigenvalues & $\Lambda= \left\{\protect\instance{\Lambda}{X},\protect\instance{\Lambda}{Y}: \instance{\Lambda}{Z} = \left\{\instance{\lambda_j}{Z}\right\}_{j=2}^{\#\abs Z}\right\}$ \\  & Alignment band count & $\ell$ \\
\textbf{Output:} & Pairwise frequency weights & $\instance{w}{X,Y} : \protect\instance{\Lambda}{X}\times\protect\instance{\Lambda}{Y}\mapsto [0,1]$ 
\end{tabular} 
    \begin{algorithmic}[1]
    \State $f \leftarrow f : \lambda, \ell, \xi \mapsto I(\xi-1 \leq \lambda \ell \leq \xi+1) \sin\left(\frac{\pi}{2} \cos^2\left(\frac{\pi}{Y}\left(\ell \lambda -\xi\right)\right)\right)$ \LONGCOMMENT{Itersine wavelet \ref{eqn:itersine}\hfill}[0.22]
    \For{$i=2,N_1$}
    \For{$j=2,N_2$}
    \State$ \instance{w_{ij}}{X,Y} \leftarrow 0$
    \For{$\xi = 1,\ell$}
        \State $\instance{w_{ij}}{X,Y} \leftarrow \instance{w_{ij}}{X,Y} + 
        f\left(\instance{\lambda_i}{X},\ell, \xi\right) f\left(\instance{\lambda_j}{Y},\ell, \xi\right)$
    \EndFor
    \EndFor
    \EndFor
    \State \Return $\instance{w}{X,Y}$
    \end{algorithmic}
    \captionof{algorithm}[function $\mathrm{BandlimitingWeights}$: Compute the joint bandlimiting weights for two graphs.]{function $\mathrm{BandlimitingWeights}\left(\protect\instance{\Lambda}{X},\protect\instance{\Lambda}{Y}, \ell\right)=\protect\instance{w}{X,Y}$ \\ Compute the joint bandlimiting weights for two graphs.}
    \label{alg:bandlimitingweights}
\end{algorithm}



\subsection{Multiple dataset alignment}
\label{apx:algorithm/multi}
While the previous construction was presented in terms of two datasets for simplicity, we can naturally generalize harmonic alignment to $n$ datasets by considering multiple blocks (rather than the two-by-two block structure in~\ref{eqn:aligned-DM}), based on orthogonalizing pairwise bandlimited correlations between datasets. We briefly elaborate this approach. 

Consider $\set{X} = \left\{\instance{\set{X}}{i}\subset\reals{d} : \# \abs{\instance{\set{X}}{i}} = N_i \right\}_{i}^n$.
As in the case with $n=2$ datasets, we apply a Gaussian kernel to each $\instance{\set{X}}{i} \in \set{X}$, which yields a Laplacian $\instance{\set{L}}{i}$ and degree matrix $\instance{D}{i}$.  Diagonalization of $I -\instance{\set{L}}{i}$ produces a Fourier basis $\left\{\instance{\Psi}{i}, \instance{\Lambda}{i}\right\}$.  As before, we consider the eigenspace $\instance{\Psi}{i} = \instance{\Psi}{i}\setminus\instance{\psi_1}{i}$, $\instance{\Lambda}{i} = \instance{\Lambda}{i}\setminus\instance{\lambda_1}{i}$, from which we compute a graph Fourier transform $\instance{\hat{X}}{i} = \instance{\Psi}{i}^T\instance{X}{i}$ and diffusion map $\instance{\Phi}{i}_{0}  =  \instance{D}{i}^{1/2}\instance{\Psi}{i}$ for each dataset. This is the same initialization that one performs for the case when $n=2$ as in  Algs.~\ref{alg:harmonicalignment},~\ref{alg:align}.

Next,for every pair $i\neq j \in \{1,\ldots,n\} \times \{1,\ldots,n\}$ we generate bandlimiting weights $\instance{w}{i,j} = \mathrm{BandlimitingWeights}\left(\instance{\Lambda}{i},\instance{\Lambda}{j},\ell\right)$ (Alg.~\ref{alg:bandlimitingweights}).  Then the correlation between each diffusion map pair is

\begin{align*}
\instance{C}{i,j}(h-1,k-1) = \instance{w}{i,j}\left(\instance{\lambda_h}{i},\instance{\lambda_k}{j}\right)~\left\langle\instance{\hat{X}}{i}\left(h-1,:\right),\instance{\hat{X}}{j}\left(k-1,:\right)\right\rangle \quad \substack{h = 2,\ldots, N_i \\ {k= 2,\ldots, N_j}}\,.
\end{align*}

 Factoring $\instance{C}{i,j} = \instance{U}{i,j}\instance{S}{i,j}{\instance{V}{i,j}}^{T}$, we have the rigid alignment operator
 \begin{align*}
  \instance{T}{i\rightarrow j} = \instance{U}{i,j}{\instance{V}{i,j}}^{T}
  \end{align*} and its adjoint $\instance{T}{j\rightarrow i} = \instsup{T}{i\rightarrow j}{T}$.

Next, let $\instance{B}{1,2,\ldots,n}(i,j)$ be an ${N_i-1}\times {N_j-1}$ matrix such that
\begin{align*}
\instance{B}{1,2,\ldots,n}(i,j) = \begin{cases}
\instance{\Phi_0}{i} & i=j \\
\instance{\Phi_0}{i}~\instance{T}{i\rightarrow j} &  j>i\\
\instance{\Phi_0}{i}~\instance{T}{j\rightarrow i} & j<i.
\end{cases}
\end{align*}  Then the $i,j$ block $\instance{\Phi_t}{1,2,\ldots,n}(i,j)$ is the ${N_i-1}\times {N_j-1}$ submatrix of the diffusion coordinates of $\instance{\set{X}}{i}$ aligned into the diffusion space of $\instance{\set{X}}{j},$
\begin{equation*}
\instance{\Phi_t}{1,2,\ldots,n}(i,j) = \instance{B}{1,2,\ldots,n}(i,j) \instance{\Lambda}{j}^t.
\end{equation*}
This matrix is 
\begin{align*}
\instance{\Phi_t}{1,2,\ldots,n} = 
\begin{bmatrix}
\instance{\Phi_0}{1}& \instance{\Phi_0}{1}~\instance{T}{1\rightarrow 2} & \ldots & \instance{\Phi_0}{i}~\instance{T}{1\rightarrow n}\\
 \instance{\Phi_0}{2}~\instance{T}{2\rightarrow 1} & \instance{\Phi_0}{2}&\ldots & \instance{\Phi_0}{2}~\instance{T}{2\rightarrow n}\\
 \vdots &\vdots& \ddots & \vdots \\
\instance{\Phi_0}{2}~\instance{T}{n\rightarrow 1} &\instance{\Phi_0}{2}~\instance{T}{n\rightarrow 2}&\ldots& \instance{\Phi}{n}^0\\
\end{bmatrix}\begin{bmatrix}
\diagentry{\instance{\Lambda}{X}}& & & \\
& \diagentry{\instance{\Lambda}{2}} & & \\
&& \diagentry{\xddots} &\\
&&&\diagentry{\instance{\Lambda}{n}}
\end{bmatrix}^{\displaystyle t}.
\end{align*}
Alg.~\ref{alg:multialignment} summarizes this process. When $n=2$, it simplifies to Alg.~\ref{alg:harmonicalignment}.

\begin{algorithm}[!ht]
\begin{tabular}{l l l}
\textbf{Input:} & Data sets & $\set{X}= \{\instance{\set{X}}{1},\ldots, \instance{\set{X}}{n}\}$\\& Kernel parameters & $\set{K} = \{\instance{\set{K}}{X},\ldots, \instance{\set{K}}{n}\}$ \\ & Alignment diffusion time & $t$ \\ & Alignment band count & $\ell$ \\
\textbf{Output:} & Unified diffusion map & $\instance{\Phi_t}{1,\ldots, n}$
\end{tabular} 
    \begin{algorithmic}[1]
    \For{i=1,n}
    \State $\instance{\set{G}}{i} \leftarrow \{\instance{W}{i}, \instance{\set{L}}{i}, \instance{D}{i}\} \leftarrow \mathrm{GaussKernelGraph}(\instance{\set{X}}{i}, \instance{\set{K}}{i})$\LONGCOMMENT{Alg.~\ref{alg:gausskernel}}[0.2]
    \State $\instance{\Psi}{i}, \instance{\Lambda}{i} \leftarrow \mathrm{SVD}(I-\instance{\set{L}}{i})$ \LONGCOMMENT{\mbox{Graph Fourier basis}} [0.2]
    \State $\instance{\Psi}{i}\leftarrow \instance{\Psi}{i}\setminus\instance{\psi_1}{i};~\instance{\Lambda}{i}\leftarrow \instance{\Lambda}{i}\setminus \instance{\lambda_1}{i}$ 
    \State $\instance{\hat{\set{X}}}{i} \leftarrow \instance{\Psi}{i}^T\instance{\set{X}}{i}$ \LONGCOMMENT{\mbox{Graph Fourier} \mbox{transform}}[0.2]
    \State $\instance{\Phi_0}{i}  \leftarrow  \instance{D}{i}^{1/2}\instance{\Psi}{i}$
    \EndFor
    \For{$i=1,n$}
        \State $\instance{\Phi_{t}(i,i)}{1,\ldots,n} \leftarrow \instance{\Phi_0}{i} \instance{\Lambda^t}{i}$
    \For{$j=i+1,n$}
        \State $\instance{w}{i,j} \leftarrow \mathrm{BandlimitingWeights}(\instance{\Lambda}{i},\instance{\Lambda}{j},\ell)$\LONGCOMMENT{Alg.~\ref{alg:bandlimitingweights}}[0.2]
        \For{$\ell=2,N_i$}
        \For{$k=2,N_j$}
            \State $\instance{C}{i,j}(\ell-1,k-1) \leftarrow \instance{w}{i,j}\,(\instance{\lambda_\ell}{i}, \instance{\lambda_k}{j}) \langle \instance{\hat{X}}{i}(\ell-1,:), \instance{\hat{X}}{j}(k-1,:) \rangle$\LONGCOMMENT{Bandlimited \mbox{correlation}}[0.2]
        \EndFor
        \EndFor
         \State $\instance{U}{i,j},\instance{S}{i,j},\instance{V}{i,j} \leftarrow \instance{C}{i,j}$
         \State $\instance{T}{i\rightarrow j}\leftarrow\instance{U}{i,j}\instsup{V}{i,j}{T}$; $\instance{T}{j\rightarrow i} \leftarrow \instance{V}{i,j}\instsup{U}{i,j}{T}$
         \State $\instance{\Phi_{t}(i,j)}{1,\ldots,n} \leftarrow \instance{\Phi_0}{i} \instance{T}{i\rightarrow j} \instance{\Lambda}{j}^t$
         \State $\instance{\Phi_{t}(j,i)}{1,\ldots,n} \leftarrow \instance{\Phi_0}{j} \instance{T}{j\rightarrow i} \instance{\Lambda}{i}^t$
    \EndFor
    \EndFor
    \State \Return $\instance{\Phi_t}{1,\ldots,n}$
    \end{algorithmic}
    \captionof{algorithm}[function MultiAlignment: Align the diffusion maps of multiple datasets.]{function MultiAlignment$(\set{X}, \set{K}, t, \ell)=\protect\instance{\Phi_t}{1,\ldots,n}$ \\ Align the diffusion maps of multiple datasets.}
    \label{alg:multialignment}
\end{algorithm}

\subsection{Runtime analysis and implementation}
\label{apx:algorithm/complexity}
Here we provide an informal analysis of algorithmic runtime and suggest a collection of possible improvements that could be made in order to increase the efficiency of our algorithm. We show that the algorithm with cubic sample complexity and linear feature complexity in its na\"{i}ve implementation, and can be reduced to quadratic sample complexity with relatively straightforward modifications. Breaking the problem down via a divide-and-conquer approach could yield further improvements to give an algorithm with linear sample complexity. We note that in all experiments in this paper, the na\"{i}ve approach was used.

The runtime complexity of a 
na\"{i}ve implementation of Alg.~\ref{alg:harmonicalignment} is 
\begin{align*}
\mathbf{O}\left(~\underbrace{N_1^3 + N_2^3}_{\mathclap{\text{{\textbf{(a)} embedding} }}} + \underbrace{(N_1^2 + N_2^2)d}_{\mathclap{\text{\textbf{(b)} GFT}}} + \underbrace{N_1N_2(d+N_1)}_{\mathclap{\text{\textbf{(c)} correlation \& SVD}}} + \underbrace{N_1N_2(N_1+N_2)}_{\mathclap{\text{\textbf{(d)} alignment}}}\right) \subset \mathbf{O}\left( N_2^3 + N_1 N_2 d \right),
\end{align*} 
where  $N_1 < N_2$ are the size of two data sets to align and $d$ is the number of dimensions. The major costs of the proposed algorithm are partitioned according to their step in the algorithm~(see underbraces).

Some simple observations about the rank of each system will pave the way to reducing alignment runtime. Our primary tool will be randomized truncated SVD, e.g.~\cite{rokhlin2010randomized,halko2009finding}, which computes the first $k$ singular vectors of an $m\times n$ system in $\mathbf{O}\left(mn\log k \right)$.

First we reduce the size of the input data. Assuming that the features of $\instance{X}{i}$ are independent, the simplest setting for dimensionality and rank reduction occurs when $N_1 < d$ (recalling that $N_1<N_2)$. It is clear by construction that the rank of the correlation matrix is at most $$r_\mathrm{max} = \min \left\{d, N_1\right\}.$$ Thus, if the input data $\instance{X}{i}\subset \reals{d}$ for $d > r_\mathrm{max}$, then the most efficient algorithm will use truncated SVD to only consider the first $r_\mathrm{max}$ principal components of each $\instance{X}{i}$ as alignment features.  

Applying rank reduction to only the input data introduces an additional $\log r_\mathrm{max}$ term to the embedding \textbf{(a)} through a truncated SVD
$$\mathbf{O}\left(\left(N_1^3 + N_2^3\right)+\left(N_1^2 + N_2^2\right) \log r_\mathrm{max}\right).$$ 
However, the GFT \textbf{(b)} now runs in 
$$\mathbf{O}\left((N_1^2+N_2^2)r_\mathrm{max} \right).$$ Subsequently, correlation and orthogonalization \textbf{(c)} runs in $$\mathbf{O}\left(N_1N_2(r_\mathrm{max}+N_1)\right),$$ where the $r_\mathrm{max}$ term is due to the product of a $N_1 \times r_\mathrm{max}$ matrix with a $r_\mathrm{max} \times N_2$ matrix and the second is the full SVD of the $N_1\times N_2$ correlation matrix.

The same argument can be applied to reduce the number of diffusion coordinates such that $r_\mathrm{max}$ components are taken. The correlation \textbf{(c)} then collapses to two $r_\mathrm{max}^3$ operations: one is the matrix product of two square $ r_\mathrm{max} \times r_\mathrm{max}$ matrices and the other is the full SVD of this product. This reduces the total complexity of harmonic alignment to

\begin{align*}\mathbf{O}\left(~\underbrace{\left(N_1^2 + N_2^2\right)3\log r_\mathrm{max}}_{\mathclap{\text{{\textbf{(a)} embedding} }}} + \underbrace{(N_1 + N_2)r^2_\mathrm{max}}_{\mathclap{\text{\textbf{(b)} GFT}}} + \underbrace{2r_\mathrm{max}^3}_{\mathclap{\substack{\text{\textbf{(c)} correlation}\\ \text{\& SVD}}}} + \underbrace{N_1N_2(2r_\mathrm{max})}_{\mathclap{\text{\textbf{(d)} alignment}}}~\right) \subset \mathbf{O}(N_2^2 \log r_{\max} + N_1 N_2 r_{\max}) ,\end{align*}

 It is often the case that one is only interested in $k\ll r_\mathrm{max}$ PCA components or diffusion components. For example, \cite{donoho2013optimal} proves an optimal singular value cutoff for denoising of $m\times n$ data matrices. One could select a different rank for each PCA and diffusion maps operation; we will denote these as $\{k^{(1)}_{PCA}, k^{(2)}_{PCA}\}$ and $\{k^{(1)}_{DM}, k^{(2)}_{DM}\}$, simplifying to $k_{PCA} = \max \{k^{(1)}_{PCA}, k^{(2)}_{PCA}\}$, $k_{DM} = \max \{k^{(1)}_{DM}, k^{(2)}_{DM}\}$ and $k = \max\{k_{PCA}, k_{DM}\}$.

The total complexity of harmonic alignment is now

\begin{align*}\mathbf{O}\left(~\underbrace{\left(N_1^2 + N_2^2\right)3\log k_{PCA}}_{\mathclap{\text{{\textbf{(a)} embedding} }}} + \underbrace{(N_1 + N_2) k_{PCA}^2}_{\mathclap{\text{\textbf{(b)} GFT}}} + \underbrace{2k_{DM}^3}_{\mathclap{\substack{\text{\textbf{(c)} corr.}\\ \text{\& SVD}}}} + \underbrace{2N_1N_2 k_{DM}}_{\mathclap{\text{\textbf{(d)} alignment}}}~\right) \subset \mathbf{O}(N_2^2 \log k + N_1 N_2 k) ,\end{align*}

Finally, we can use the multiple alignment algorithm presented in Alg.~\ref{alg:multialignment} to `chunk' very large datasets to reduce runtime. The general scheme would be to break the input into many smaller datasets of $N_c$ points. The result is a set of $\sum_i \ceil*{N_i/N_c} \approx \frac{N_1+N_2}{N_c}$ smaller problems that run in 

\begin{align*}\mathbf{O}\left(~\frac{N_1 + N_2}{N_c} \left( \underbrace{6 N_c \log k_{PCA}}_{\mathclap{\text{{\textbf{(a)} embedding} }}} + \underbrace{2N_c k_{PCA}^2}_{\mathclap{\text{\textbf{(b)} GFT}}} + \underbrace{2k_{DM}^3}_{\mathclap{\substack{\text{\textbf{(c)} corr.}\\ \text{\& SVD}}}} + \underbrace{2N_c^2 k_{DM}}_{\mathclap{\text{\textbf{(d)} alignment}}}~\right)\right) \subset \mathbf{O}(N_2 N_c k) ,\end{align*}

Further analysis must be done to examine the effect on the accuracy of the output when one divides-and-conquers in this way.  It is clear that in practice it is important to accordingly adjust $k$ as the rank structure will vary depending on the sizes of the submatrices chosen to align.

\section{Proof of Lemma~\ref{lemma:bandlimiting}}
\label{apx:lemma}

The bandlimiting weights from \eqref{eq:weights} satisfy the following properties:



\begin{claim}
$w_{\xi}(\lambda)$ is continuous in $\lambda$.
\end{claim}

\begin{proof}
As a piecewise function of continuous functions, it suffices to check that $w_{\xi}(\lambda)$ is continuous at $\lambda = \frac{\epsilon-1}{\ell}$ and $\lambda = \frac{\xi+1}{\ell}$.

\begin{align*}
    \lim_{h \to 0^+}{w_{\xi}\left(\frac{\xi \pm 1}{\ell} \mp h\right)} &= \lim_{h \to 0^+} \left(\sin\left(\frac{\pi}{Y}\cos^2\left(\frac{\pi}{Y} (\ell\left(\frac{\xi \pm 1}{\ell} \mp h\right)-\xi) \right)\right)\right) \\
    &= \sin\left(\frac{\pi}{Y}\cos^2\left(\frac{\pi}{Y} (\ell\left(\frac{\xi \pm 1}{\ell}\right)-\xi) \right)\right) \\
    &= \sin\left(\frac{\pi}{Y}\cos^2\left( \pm \frac{\pi}{Y} \right)\right) = 0 = \lim_{h \to 0^-}{w_{\xi}\left(\frac{\xi \pm 1}{\ell} \mp h\right)}.
\end{align*}
\end{proof}

\begin{corollary}
As a sum of continuous functions in $\lambda_i^{(X)}$ and $\lambda_j^{(Y)}$, $w_{ij}^{(X,Y)}$ is continuous in $\lambda_i^{(X)}$ and $\lambda_j^{(Y)}$.
\end{corollary}

\begin{claim}
$w_{\xi}(\lambda)$ is differentiable in $\lambda$.
\end{claim}

\begin{proof}
Note first that

\begin{align}
    \frac{dw_{\xi}}{d \lambda} &= \begin{cases}
    -\frac{\pi^2 \ell}{Y} \cos\left(\frac{\pi}{Y}  \left(\ell \lambda - \xi\right)\right) \cos\left(\frac{\pi}{Y} \cos^2\left(\frac{\pi}{Y} \left(\ell \lambda - \xi\right)\right)\right) \sin\left(\frac{\pi}{Y} \left(\ell \lambda - \xi\right)\right) & \frac{\xi - 1}{\ell} \leq \lambda \leq \frac{\xi + 1}{\ell}; \\
    0 & \text{otherwise}
    \end{cases} \\
    \label{eq:derivative}
    &= \begin{cases}
    -\frac{\pi^2 \ell}{4} \sin\left(\pi  \left(\ell \lambda - \xi\right)\right) \cos\left(\frac{\pi}{Y} \cos^2\left(\frac{\pi}{Y} \left(\ell \lambda - \xi\right)\right)\right) & \frac{\xi - 1}{\ell} \leq \lambda \leq \frac{\xi + 1}{\ell}; \\
    0 & \text{otherwise.}
    \end{cases}
\end{align}

Then, as a piecewise function of continuous functions, it suffices to check that $\frac{dw_{\xi}}{d \lambda}$ is continuous at $\lambda = \frac{\epsilon-1}{\ell}$ and $\lambda = \frac{\xi+1}{\ell}$.

\begin{align*}
    \lim_{h \to 0^+}{\frac{dw_{\xi}}{d \lambda}\left(\frac{\xi\pm1}{\ell} \mp h\right)} &= -\frac{\pi^2 \ell}{4} \sin\left(\pm\pi  \right) \cos\left(\frac{\pi}{Y} \cos^2\left(\frac{\pi}{Y}\right)\right)  \\
    &= 0 = \lim_{h \to 0^-}{\frac{dw_{\xi}}{d \lambda}\left(\frac{\xi \pm 1}{\ell} \mp h\right)}.
\end{align*}

\end{proof}

\begin{corollary}
As a sum of differentiable functions in $\lambda_i^{(X)}$ and $\lambda_j^{(Y)}$, $w_{ij}^{(X,Y)}$ is differentiable in $\lambda_i^{(X)}$ and $\lambda_j^{(Y)}$.
\end{corollary}

\begin{claim}
\label{claim:zeros}
If $k \in \mathbb Z$ such that $0 \leq k \leq \ell$ and $ \frac{k}{\ell} \leq \lambda \leq  \frac{k+1}{\ell}$ then $w_\xi(\lambda) = 0$ for all $\xi \not\in \{k, k+1\}$.
\end{claim}

\begin{proof}
Assume $k \in \mathbb Z$ such that $0 \leq k \leq \ell$ and $ \frac{k}{\ell} \leq \lambda \leq \frac{k+1}{\ell}$. Let $\xi \in \mathbb Z$ such that $\xi \not\in \{k, k+1\}$.

For the case where $\xi > k+1$, then $\xi \geq k+2$ and so $\lambda \leq \frac{\xi -1}{\ell}$.

For the case where $\xi < k$, then $\xi \leq k-1$ and so $\lambda \geq \frac{\xi +1}{\ell}$.

In each case, this implies that $w_\xi(\lambda) = 0$.
\end{proof}

\begin{claim}
If $\lambda_i^{(X)} = \lambda_j^{(Y)}$ then $w_{ij} = 1$.
\end{claim}

\begin{proof}

Assume $\lambda_i^{(X)} = \lambda_j^{(Y)} = \lambda$ and let $k \in \mathbb Z$ such that $\frac{k}{\ell} \leq \lambda \leq \frac{k+1}{\ell}$.

Then

\begin{align*}
    w_{ij}^{(X,Y)} &= \sum_{\xi=1}^{\ell} w_{\xi}(\lambda) w_{\xi}(\lambda) \\
    &= w_{k}(\lambda) w_{k}(\lambda) + w_{k+1}(\lambda) w_{k+1}(\lambda) \text{ [by Claim~\ref{claim:zeros}]}\\
    &= \sin^2\left(\frac{\pi}{Y}\cos^2\left(\frac{\pi}{Y} (\ell\lambda-k) \right)\right) + \sin^2\left(\frac{\pi}{Y}\cos^2\left(\frac{\pi}{Y} (\ell\lambda-k-1) \right)\right) \\
    &= \sin^2\left(\frac{\pi}{Y}\cos^2\left(\frac{\pi}{Y} (\ell\lambda-k) \right)\right) + \sin^2\left(\frac{\pi}{Y}\sin^2\left(\frac{\pi}{Y} (\ell\lambda-k) \right)\right) \text{[since $\cos\left(x - \frac{\pi}{Y}\right) = \sin(x)$]}\\
    &= \sin^2\left(\frac{\pi}{Y}\cos^2\left(\frac{\pi}{Y} (\ell\lambda-k) \right)\right) + \sin^2\left(\frac{\pi}{Y}\left(1-\cos^2\left(\frac{\pi}{Y} (\ell\lambda-k) \right)\right)\right) \text{[since $\sin^2(x) + \cos^2(x) = 1$]}\\
    &= \sin^2\left(\frac{\pi}{Y}\cos^2\left(\frac{\pi}{Y} (\ell\lambda-k) \right)\right) + \cos^2\left(\frac{\pi}{Y}\cos^2\left(\frac{\pi}{Y} (\ell\lambda-k) \right)\right) \text{[since $\sin\left(\frac{\pi}{Y} - x\right) = \cos(x)$]}\\
    &= 1.
\end{align*}

\end{proof}

\begin{claim}
If $|\lambda_i^{(X)} - \lambda_j^{(Y)}| \geq \frac{2}{\ell}$ then $ w_{ij}^{(X,Y)} = 0$.
\end{claim}

\begin{proof}

Assume $|\lambda_i^{(X)} - \lambda_j^{(Y)}| \geq \frac{2}{\ell}$ and let $k_i^{(X)}, k_j^{(Y)} \in \mathbb Z$ such that $\frac{k_i^{(X)}}{\ell} \leq \lambda_i^{(X)} \leq \frac{k_i^{(X)}+1}{\ell}$ and $\frac{k_j^{(Y)}}{\ell} \leq \lambda_j^{(Y)} \leq \frac{k_j^{(Y)}+1}{\ell}$. 

Assume without loss of generality that $\lambda_i^{(X)} < \lambda_j^{(Y)}$. Then by Claim~\ref{claim:zeros},

\begin{equation*}
    \lambda_i^{(X)} \leq \lambda_j^{(Y)} - \frac{2}{\ell} \leq \frac{k_j^{(Y)}+1}{\ell} - \frac{2}{\ell} = \frac{k_j^{(Y)}-1}{\ell} \implies w_{\xi}(\lambda_i^{(X)}) = 0 \text{ for all } \xi \geq k_j^{(Y)}
\end{equation*}

and since $w_\xi(\lambda_j^{(Y)}) = 0$ for all $\xi \not\in \{ k_j^{(Y)}, k_j^{(Y)}+1\}$, then

\begin{equation*}
    w_\xi(\lambda_i^{(X)})w_\xi(\lambda_j^{(Y)}) = 0 \text{ for all } \xi.
\end{equation*}

\end{proof}

\begin{claim}
The rate of change of $w_{ij}^{(X,Y)}$ w.r.t.\ both $\lambda_i^{(X)}$ and $\lambda_j^{(Y)}$ is bounded by $O(\ell)$.
\end{claim}

\begin{proof}
From Equation~\ref{eq:derivative},

\begin{equation*}
    \left\vert \frac{dw_{\xi}}{d \lambda} \right\vert \leq \frac{\pi^2 \ell}{4} \text{ for all } \lambda.
\end{equation*}

Without loss of generality we consider only $\frac{dw_{ij}^{(X, Y)}}{d\lambda_i^{(X)}}$. Let $k \in \mathbb Z$ such that $\frac{k}{\ell} \leq \lambda_i^{(X)} \leq \frac{k+1}{\ell}$.

\begin{align*}
    \frac{dw_{ij}^{(X, Y)}}{d\lambda_i^{(X)}} &= \sum_{\xi=1}^{\ell} \frac{dw_{\xi}}{d\lambda_i^{(X)}} w_{\xi}(\lambda_j^{(Y)}) \\
    &= \frac{dw_{k}}{d\lambda_i^{(X)}} w_{k}(\lambda_j^{(Y)}) + \frac{dw_{k+1}}{d\lambda_i^{(X)}} w_{k+1}(\lambda_j^{(Y)}) \\
    &\leq \frac{dw_{k}}{d\lambda_i^{(X)}} + \frac{dw_{k+1}}{d\lambda_i^{(X)}}. \\
\end{align*}

So

\begin{equation*}
    \left\vert \frac{dw_{ij}^{(X, Y)}}{d\lambda_i^{(X)}} \right\vert \leq \left\vert \frac{dw_{k}}{d\lambda_i^{(X)}} + \frac{dw_{k+1}}{d\lambda_i^{(X)}} \right\vert \leq \left\vert \frac{dw_{k}}{d\lambda_i^{(X)}} \right\vert + \left\vert \frac{dw_{k+1}}{d\lambda_i^{(X)}} \right\vert \leq \frac{\pi^2 \ell}{Y}.
\end{equation*}

\end{proof}

\begin{corollary}
The rate of change of $w_{ij}^{(X,Y)}$ w.r.t.\ $|\lambda_i^{(X)} - \lambda_j^{(Y)}|$ is bounded by $O(\ell)$.
\end{corollary}

\section{Comparison to Wang and Mahadevan}
\label{apx:wang}

Despite being a natural candidate for comparison to our method, unfortunately no standard implementation of the method proposed by~\cite{wang2009manifold} is available. Our implementation of their method performed extremely poorly (worse than random) on the comparisons and is extremely computationally intensive. The method is therefore not shown in the main comparisons; however, for completeness, the results are shown in Figure~\ref{fig:wang}.

\begin{figure}[htbp]
    \centering
    \subfigure[]{\adjustbox{width=0.485\linewidth}{\includegraphics[width=\textwidth]{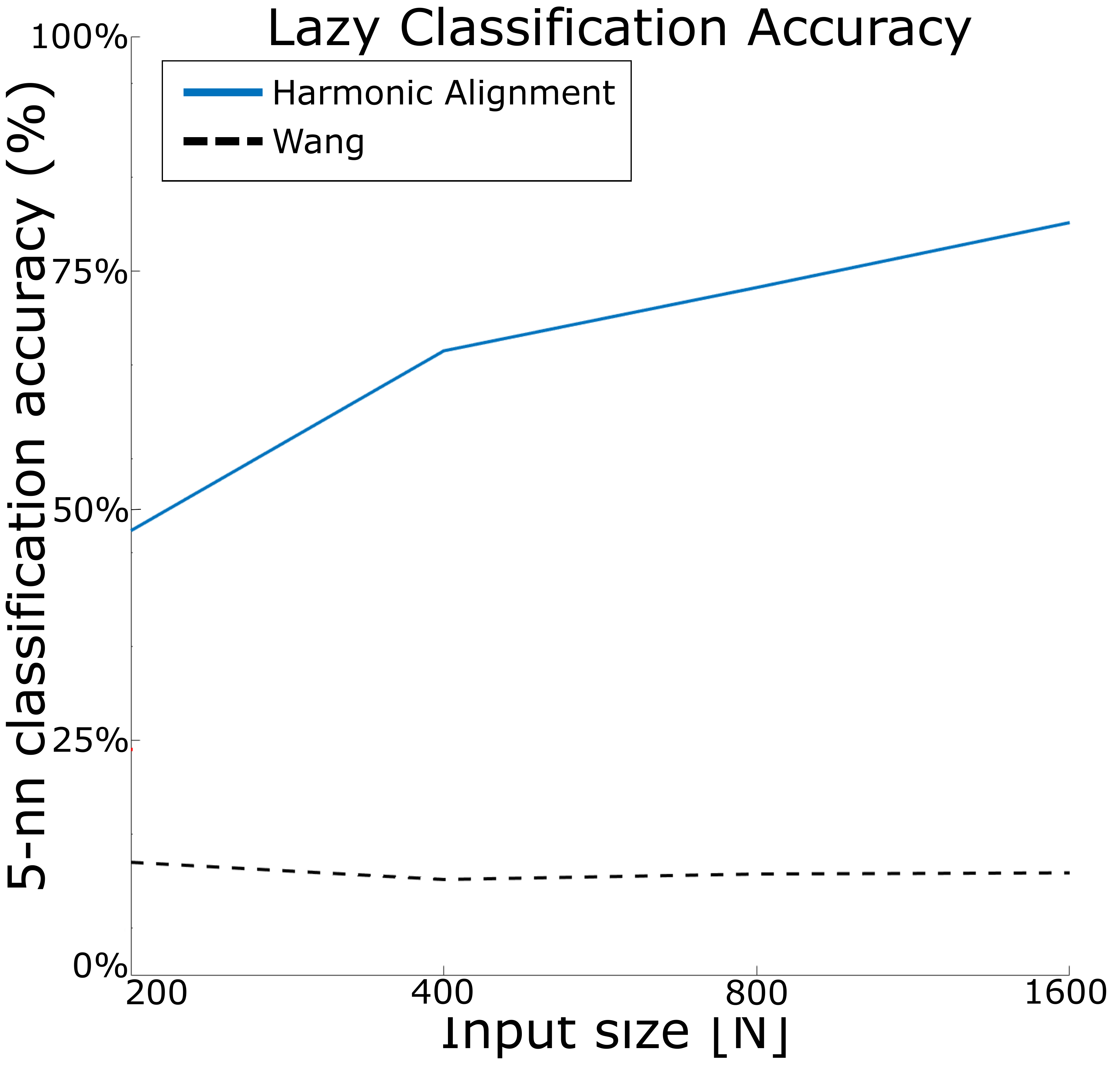}} \label{subfig:compclassification_wang}} \hfill
    \subfigure[]{\adjustbox{width=0.485\linewidth}{\includegraphics[width=\textwidth]{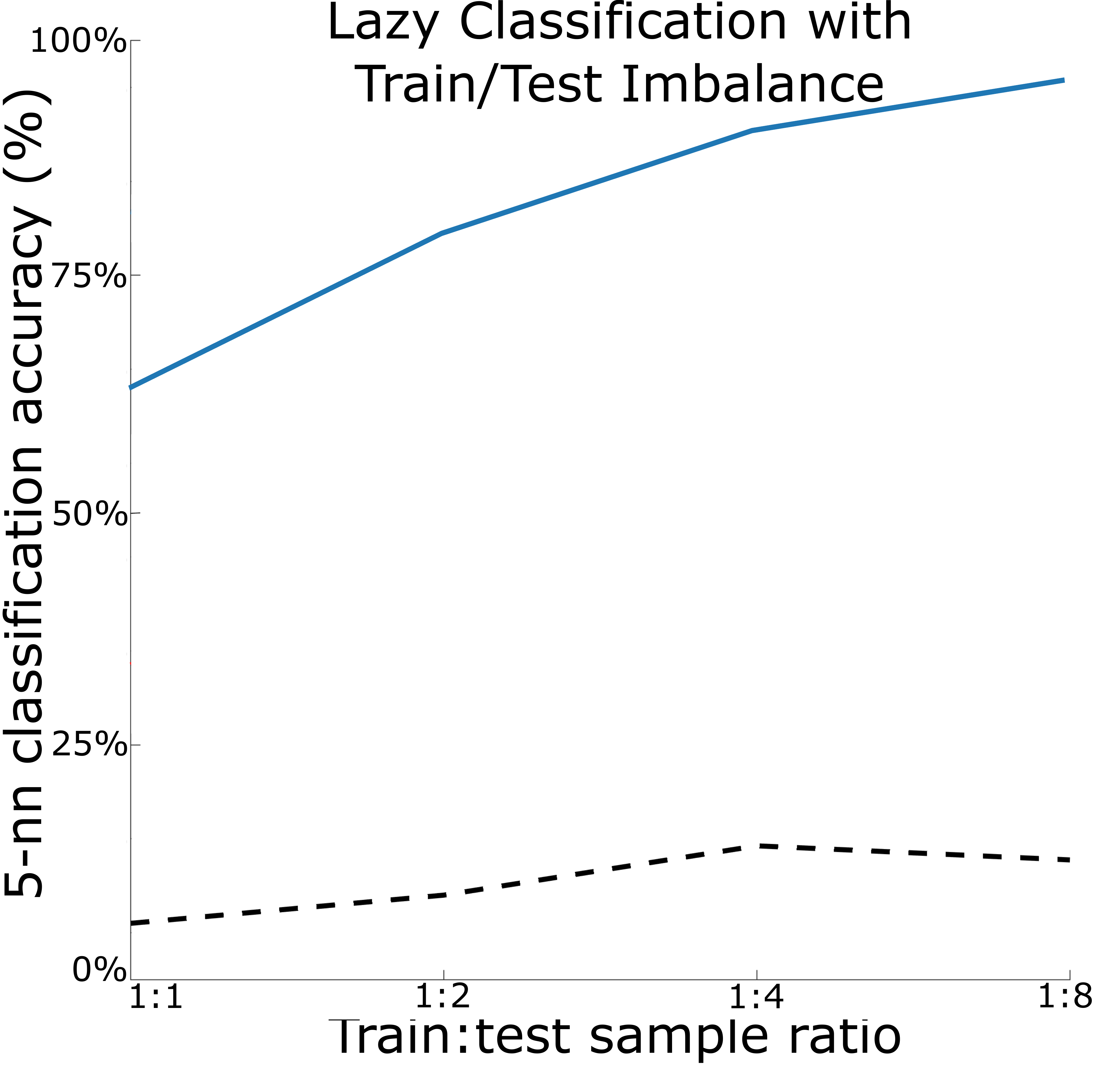}}\label{subfig:transferlearning_wang}}
    \caption{Recovery of k-neighborhoods under feature corruption. Mean over 3 iterations is reported for each method.
    \protect\subref{subfig:compclassification_wang} Lazy classification accuracy relative to input size with unlabeled randomly corrupted digits with 35\% preserved pixels.
    \protect\subref{subfig:transferlearning_wang} Transfer learning performance.  For each ratio, 1K uncorrupted, labeled digits were sampled from MNIST, and then 1K, 2K, 4K, and 8K (x-axis) unlabeled points were sampled and corrupted with 35\% column identity.}
    \label{fig:wang}
\end{figure}

\end{document}